\documentclass[letterpaper, 10 pt, conference]{ieeeconf}  

\IEEEoverridecommandlockouts                              

\usepackage{amsmath,amssymb}
\usepackage{amsfonts}
\usepackage{mathrsfs}
\usepackage{color}
\usepackage{graphicx}

\usepackage{array,booktabs} 
\usepackage{amsmath,amssymb} 
\usepackage{stfloats} 
\usepackage{psfrag} 
\usepackage{color}

\definecolor{box_color}{rgb}{.8,.8,.8}

\usepackage[utf8]{inputenc}
\usepackage[T1]{fontenc}
\usepackage{lmodern}

\newtheorem{lemma}{Lemma}
\newtheorem{prop}{Proposition}
\newtheorem{corollary}{Corollary}
\newtheorem{fact}{Fact}

\newtheorem{remark}{Remark}

\newtheorem{assumption}{Assumption}
\newtheorem{ass}{C.}

\def\begcen{\begin{center}}
\def\endcen{\end{center}}

\newcommand{\dq}[1]{{\partial_{\bfq}}{#1}}

\def\L2{{\cal L}_2}
\def\L2e{{\cal L}_{2e}}

\def\rea{\mathbb{R}}

\def\dq{\dot{q}}

\def\dtheta{\dot{\theta}}
\def\ddtheta{\ddot{\theta}}
\def\dotp{\dot{p}}
\def\ds{\dot{s}}
\def\dz{\dot{z}}
\def\phiv{\varphi}
\def\dphiv{\dot{\varphi}}

\def\begequarr{\begin{eqnarray}}
\def\endequarr{\end{eqnarray}}
\def\begequarrs{\begin{eqnarray*}}
\def\endequarrs{\end{eqnarray*}}
\def\begarr{\begin{array}}
\def\endarr{\end{array}}
\def\begequ{\begin{equation}}
\def\endequ{\end{equation}}

\def\begdes{\begin{description}}
\def\enddes{\end{description}}
\def\begenu{\begin{enumerate}}
\def\begite{\begin{itemize}}
\def\endite{\end{itemize}}
\def\endenu{\end{enumerate}}

\def\lef[{\left[\begin{array}}
\def\rig]{\end{array}\right]}

\def\begcen{\begin{center}}
\def\endcen{\end{center}}
\def\begrem{\begin{remark}\rm}
\def\endrem{\end{remark}}
\def\begassum{\begin{assumption}}
\def\endassum{\end{assumption}}
\def\begassums{\begin{assumption*}}
\def\endassums{\end{assumption*}}
\def\begassu{\begin{ass}}
\def\endassu{\end{ass}}
\def\beglem{\begin{lemma}}
\def\endlem{\end{lemma}}
\def\begcor{\begin{corollary}}
\def\endcor{\end{corollary}}
\def\begfac{\begin{fact}}
\def\endfac{\end{fact}}

\title{\LARGE \bf
Force and state-feedback control for robots with non-collocated environmental and actuator forces*
}

\author{Alejandro Donaire$^{1}$, Luigi Villani$^{2}$, Fanny Ficuciello$^{2}$, Juan Tomassini$^{3}$, Bruno Siciliano$^{2}$
\thanks{*The research leading to these results has been supported by the RoDyMan project, which has received funding from the European Research Council FP7 Ideas under Advanced Grant agreement number 320992. The authors are solely responsible for the content of this manuscript. J. Tomassini acknowledges the SeCyT-UNR (the Secretary for Science and Technology of the National University of Rosario) and ANPCyT for their financial support through projects PID-UNR 1ING502 and FONCyT PICT 2012 Nr. 2471.}
\thanks{$^{1}$A. Donaire is with School of Engineering, The University of Newcastle, 2308 Callaghan, NSW, Australia,
        {\tt\small \{alejandro.donaire\}@newcastle.edu.au}}%
\thanks{$^{2}$L. Villani, F. Ficuciello and B. Siciliano are with PRISMA Lab, Department of Electrical Engineering and Information Technology, University of Naples Federico II,  80125 Naples, Italy,
        {\tt\small \{luigi.villani,fanny.ficuciello,bruno.siciliano\}} {\tt\small @unina.it}}%
\thanks{$^{3}$J. Tomassini is with Lab of Automation and Control, School of Electronic Engineering, National University of Rosario, S2000EKE Rosario, Argentina,
        {\tt\small tomajuan@fceia.unr.edu.ar}}%
}

\begin{document}

\maketitle
\thispagestyle{empty}
\pagestyle{empty}

\begin{abstract}
In this paper, we present an impedance control design for multi-variable linear and nonlinear robotic systems. The control design considers force and state feedback to improve the performance of the closed loop. Simultaneous feedback of forces and states allows the controller for an extra degree of freedom to approximate the desired impedance port behaviour. A numerical analysis is used to demonstrate the desired impedance closed-loop behaviour.
\end{abstract}
\section{INTRODUCTION}
One challenging control design for robotic systems is that where the physical interaction with the environment or humans must be shaped into a desired behaviour of a power port---the interaction point at which the robot and the environment share a force and a velocity. This control design problem is called \emph{impedance control}. Since its early developments (see e.g. \cite{Hogan1985,Chiaverini1999,Siciliano1999}), the research interest on shaping the impedance of the robot from the interaction port has increased as new emerging applications require a better handling of those interactions \cite{Buerger2007,Schaffer2007,Calanca2017,Magrini2015,Ficuciello2015}.  

In general, large part of the control designs for robots-human/environment physical interactions use some kind of force feedback and control algorithms to ensure that the impedance seen from the interaction port of the robot approximates a desired impedance. In most cases, the target impedance is a mass-spring-damper system behaviour. In addition, the controller is designed to ensure passivity of the closed loop, and therefore stability when the robot interacts with an environment that is also passive (see e.g. \cite{Colgate1989}).

As point it out in \cite{Ott2008}, the classical approach to impedance control design neglects the joint elasticity of the robot, and the controller performance usually deteriorates when the design is directly applied to flexible joint robots. The underactuated nature of the flexible joint robots pose a challenging problem since the interaction and actuators forces are non-collocated. The work in \cite{Schaffer2007,Ott2008} proposes  an impedance control design for flexible joint robots using the passivity-based approach (see \cite{Ortega2017} for a survey on this topic). That control design assumes that the collocated states (motor side) and joint torques, and their first derivatives, are available as information to implement the control law.

In this paper, we propose to use both force feedback and state feedback to design an impedance controller for robots with flexible joints. The controller is designed for both linear and nonlinear case (constant mass matrix, and configuration dependent mass matrix). In both robot dynamic models, we prove that the closed loop can be written in the form of a mechanical system, and the inertia, stiffness and damping matrices can be shaped by a suitable selection of the control gains. We also prove that the closed loop preserves the passivity properties at the interaction port of the robot. The simultaneous feedback of both forces and states provides extra degree of freedom to select the controller gains and therefore more freedom to tune the closed-loop impedance such that it resembles the desired impedance.

The rest of the paper is organised as follows. In Section~\ref{probform}, we formulate the control problem. In Section~\ref{mainresult}, we present the main result of the paper, for both linear and nonlinear cases. Supported by numeral analysis, we discuss the performance of the closed loop in Section~\ref{secna}. The conclusions are presented in Section~\ref{conclusion}.

\section{PROBLEM FORMULATION} \label{probform}
Consider the robot dynamics described by the Euler-Lagrange equations
\begequarr
\hspace{-6mm} M(q) \ddot{q}\!+C(q,\dot{q}) \dot{q} \!-\! K (\theta-q) \!-\! D ( \dot{\theta}-\dot{q}) \!+\! {\nabla_q V(q)} \hspace{-3mm} &=& \hspace{-2mm} \tau_e \label{el1} \\
J \ddtheta \!+\! K (\theta-q) \!+\! D (\dot{\theta}-\dot{q}) \hspace{-3mm}&=&\hspace{-2mm} \tau \label{el2}
\endequarr
with $\dim(q)=n$, $\dim(\theta)=n$, $\tau_e \in \rea^n$, $\tau \in \rea^n$, $M:\rea^n \to \rea^{n\times n}$, $C:\rea^{n\times n} \to \rea^{n\times n}$ is the Coriolis matrix, $V: \rea^n \to \rea$ represents the potential energy due to gravity, and $J$, $K$ and $D$ are positive constant matrices of adequate dimension. 

The objective is to design a control law $\tau=\alpha(q,\dq,\theta,\dtheta,\tau_e)$ that shapes the mass, damping and compliance matrices of the closed loop, and simultaneously preserves the passivity properties of the open loop with input $\tau_e$ and output $\dq$. It is desirable that the dynamics of the closed loop approximates the target dynamics
\begequ \label{targdyn}
M_d \ddot{q} + K_d \dot{q} + D_d q = \tau_e,
\endequ 
where $M_d$, $K_d$ and $D_d$ are the desired mass, stiffness and damping matrices.

For the linear case, the control design can be posed as to find a control law that satisfies two objectives:
\begin{itemize}
\item[\bf O1.] The closed-loop transfer function (admittance) $Y(s)= \frac{ \mathcal{L}[\dot q(t)] }{\mathcal{L}[\tau_e(t)]}$ is positive real (i.e. passive).
\item[\bf O2.] The admittance $Y(s)$ approximates the target admittance
\begequ \label{targadm}
Y_d(s)=\frac{s}{M_d s^2 + K_d s + D_d} ,
\endequ 
\end{itemize}
\section{MAIN RESULT} \label{mainresult}
In this section, we present a control design that achieves the control objectives formulated in  the Section \ref{probform}.
\subsection{Port-Hamiltonian form}
We develop our control design using the port-Hamiltonian framework, instead of its equivalent Euler-Lagrange. Therefore, we formulate the dynamics in the port-Hamiltonian (pH) form. We use the Legendre transformation $p=M(q)\dq$ and $s=J\dtheta$ to write the open-loop system \eqref{el1}-\eqref{el2} in the pH form as follows
\begequarr
\left[ \begarr{c} \dot q \\ \dot \theta \\ \dot p \\ \dot s \endarr \right] \hspace{-3mm}&=& \hspace{-3mm} \left[ \begarr{cccc} 0_{n\times n} & 0_{n\times n} & I_{n\times n} & 0_{n\times n} \\ 0_{n\times n} & 0_{n\times n} & 0_{n\times n} & I_{n\times n} \\ -I_{n\times n} & 0_{n\times n} & -D & D \\ 0_{n\times n} & -I_{n\times n} & D & -D \endarr \right] \nabla H_{\tt OL} \nonumber \\
&& + \left[ \begarr{c} 0_{n\times 1} \\ 0_{n\times 1} \\ \tau_e \\ \tau \endarr \right], \label{phol}
\endequarr
with the open-loop total energy
\begequarr
H_{\tt OL} &=& \frac 12 p^\top M^{-1}(q) p + \frac 12 s^\top J^{-1} s + \frac 12 (\theta-q)^\top K  \nonumber \\
&& \times (\theta-q) + {V(q)}
\endequarr
\subsection{Control design for linear multivariable systems.}
We consider first the linear case for the dynamics \eqref{phol}, and thus we make the following assumption:
\begin{itemize}
\item[{\bf A1.}] The mass matrix is constant, i.e. $M(q)=M$.
\end{itemize}

Under assumption {\bf A1}, the system \eqref{phol} (equivalently \eqref{el1} and \eqref{el2}) takes the pH form as follows
\begequarr
\dq &=& M^{-1} p \label{mechl1}\\ 
\dtheta &=& J^{-1} s \label{mechl2} \\
\dotp &=& - \nabla_q V(q) + K (\theta-q) + D (J^{-1}s-M^{-1}p) \nonumber \\
&& +\tau_e \label{mechl3}  \\
\ds &=& -K (\theta-q) - D (J^{-1}s-M^{-1}p) +\tau \label{mechl4}
\endequarr
We now present the first result, which proposes a controller to achieve the control objective required in the problem formulation.
\begin{prop} \label{proplin} Consider the LTI MIMO system \eqref{phol}, equivalently \eqref{el1}-\eqref{el2}, in closed loop with the control law
\begequarr \label{impcontrlin}
\tau=K_F \; \tau_e - K_G \; \tau_{a} - K_F \nabla_qV(q) + K_H \; \tau_u
\endequarr
where
\begin{equation}\label{taua}
\tau_{a}=  K (\theta-q) + D( \dtheta -\dq),
\end{equation}
and $K_F$, $K_G$ and $K_H$ are constant matrices that satisfy
\begequarr
K_F &=& -JK^{-1} (K_e-K) M^{-1} \label{kflin},\\
K_G &=& \left[ JK^{-1}K_e J_e^{-1} - K_F -I_n \right] \label{kglin},\\
K_H &=& J K^{-1}K_eJ_e^{-1}\label{khlin},
\endequarr
and $\tau_u$ is an input available for further control loops.
The constant matrices $J_e=J_e^\top>0$ and $K_e=K_e^\top>0$ are free to be chosen such that $D_e=DK^{-1}K_e$ is symmetric and positive definite. Then, under assumption {\bf A1}, the following statements hold true:
\begite
\item[$i).$] The closed-loop dynamics can be written in pH form as follows
\begequarr 
\left[ \begarr{c} \dq \\ \dphiv \\ \dotp \\ \dz \endarr \right] \hspace{-3mm} &=& \hspace{-3mm} \left[ \begarr{cccc} 0_{n\times n} & 0_{n\times n} & I_{n\times n} & 0_{n\times n} \\ 0_{n\times n} & 0_{n\times n} & 0_{n\times n} & I_{n\times n} \\ -I_{n\times n} & 0_{n\times n} & -D_e & D_e \\ 0_{n\times n} & -I_{n\times n} & D_e & -D_e \endarr \right] \nabla H_{\tt CL} \nonumber \\
&& + \left[ \begarr{c} 0_{n\times 1} \\ 0_{n\times 1} \\ \tau_e \\ \tau_u \endarr \right], \label{phcl}
\endequarr
with
\begequarr 
H_{\tt CL}&=&\frac 12 p^\top M^{-1} p + \frac 12 z^\top J_e^{-1} z + \frac 12 (\phiv-q)^\top K_e \nonumber \\
&& \times (\phiv-q)+ V(q) . \label{hcl}
\endequarr

The new generalised position $\phiv$ and momentum $z$ are obtained using the transformations
\begequarr
\hspace{-8mm} \phiv \hspace{-2mm}&=& \hspace{-2mm} K_e^{-1} (K_e-K)q+K_e^{-1}K\theta  \label{coc1} \\
\hspace{-8mm} z \hspace{-2mm}&=& \hspace{-2mm} J_eK_e^{-1}(K_e-K)M^{-1}p+J_eK_e^{-1}KJ^{-1}s \label{coc2}
\endequarr
\item[$ii).$] The closed-loop dynamics is passive with inputs $(\tau_e,\tau_u)$, outputs $(\dq,\dphiv)$ and storage function $H_{\tt CL}$.
\endite
\end{prop}

\begin{proof}
Note that under assumption {\bf A1}, the system \eqref{phol} can be written as the dynamics \eqref{mechl1}-\eqref{mechl4}. To prove the claim in $i)$, we will show that the dynamics \eqref{phol} with the control law \eqref{impcontrlin} matched the closed-loop dynamics \eqref{phcl}. First, we show that the first line of \eqref{phcl} is
\begequarr
\dot q &=& M^{-1} p =\nabla_p H_{\tt CL}, \label{ee1}
\endequarr
which exactly matches \eqref{mechl1}.

We consider now the change of coordinates \eqref{coc1}, and we obtain the dynamics of $\phiv$ by computing the derivative of \eqref{coc1} with respect to time as follows
\begequarr
\dot \phiv &=& K_e^{-1} (K_e-K) \dot q+K_e^{-1}K \dot \theta \nonumber \\
&=&  K_e^{-1}(K_e-K)M^{-1}p+ K_e^{-1}KJ^{-1}s  \nonumber \\
&=&J_e^{-1} \left[ J_eK_e^{-1}(K_e-K)M^{-1}p+J_eK_e^{-1}KJ^{-1}s \right] \nonumber \\
&\equiv& J_e^{-1} z = \nabla_z H_{\tt CL}, \label{ee2}
\endequarr
which is the second line of \eqref{phcl}.

Now, from \eqref{mechl3} and using \eqref{coc1} and \eqref{coc2}, we obtain
\begequarr
\dotp &=& -\nabla_qV(q)  + K (\theta-q) + D (J^{-1}s-M^{-1}p) +\tau_e \nonumber  \\
 &\equiv&  -\nabla_qV(q) + K_e (\phiv-q) + D K^{-1} K_e ( J_e^{-1}z \nonumber \\
 && -M^{-1}p) +\tau_e \nonumber  \\
&\equiv& -\nabla_qV(q) + K_e (\phiv-q) + D_e  (J_e^{-1}z-M^{-1}p ) +\tau_e \nonumber \\
&\equiv& - \nabla_q H_{\tt CL} - D_e \nabla_p H_{\tt CL} + D_e \nabla_z H_{\tt CL}  + \tau_e, \label{ee3}
\endequarr
which is the third line of \eqref{phcl}.

Similarly, we obtain the dynamics of $z$ by computing the derivative of \eqref{coc2} with respect to time as follows
\begequarr
\dot z &=& J_eK_e^{-1}(K_e-K)M^{-1} \dot p+J_eK_e^{-1}KJ^{-1} \dot s  \nonumber  \\
 &=&J_eK_e^{-1}(K_e-K)M^{-1} \left[ -\nabla_qV(q) + K (\theta-q) \right. \nonumber \\
 && \left. + D (J^{-1}s-M^{-1}p) +\tau_e  \right] + J_eK_e^{-1}KJ^{-1} \nonumber \\
 &&  \times \left[ -K (\theta-q) - D (J^{-1}s-M^{-1}p) +\tau \right]   \nonumber  \\
 &=&J_eK_e^{-1}(K_e-K)M^{-1} \left[ -\nabla_qV(q) + K (\theta-q) \nonumber \right. \\
 && \left. + D (J^{-1}s-M^{-1}p) +\tau_e  \right] + J_eK_e^{-1}KJ^{-1} \nonumber \\
 && \times \left[ -K (\theta-q) - D (J^{-1}s-M^{-1}p) \right] + J_eK_e^{-1}  \nonumber  \\
&&  \times KJ^{-1} \left\{  K_F \left[\tau_e -\nabla_qV(q) \right] + K_G \left[- K (\theta-q)  \right. \right. \nonumber \\
&& \left. \left. - D(\dtheta-\dq ) \right] + K_H \; \tau_u   \right\}   \nonumber  \\
&=& \left[ J_eK_e^{-1}(K_e-K)M^{-1} - J_eK_e^{-1}KJ^{-1} - J_eK_e^{-1} \right. \nonumber \\
&& \left. \times K J^{-1} K_G \right]  \left[  K (\theta-q) + D (J^{-1}s-M^{-1}p)  \right] \nonumber \\
 && +  \left[ J_eK_e^{-1}(K_e-K)M^{-1} + J_eK_e^{-1}KJ^{-1} K_F  \right] \nonumber \\
 && \times  \left[\tau_e -\nabla_qV(q) \right]  \;+\; \tau_u  \nonumber  \\
&=&  -K (\theta-q) - D (M^{-1}p-J^{-1}s) \;+\; \tau_u  \nonumber \\
&\equiv&  -K_e (\phiv-q) - D_e(J_e^{-1}z-M^{-1}p) \;+\; \tau_u \nonumber \\
&\equiv&  - \nabla_\phiv H_{\tt CL} + D_e \nabla_p H_{\tt CL} - D_e \nabla_z H_{\tt CL} \;+\; \tau_u \label{ee4}
\endequarr
which is the last line of \eqref{phcl}. Finally, we obtain the dynamics \eqref{phcl} by writing \eqref{ee1}-\eqref{ee4} in matrix form.

To prove the passivity property claimed in $ii)$, we use the Hamiltonian \eqref{hcl} as storage function, and we compute its time derivative as follows
\begequarr 
\dot H_{\tt CL} &=&[\nabla_p H_{\tt CL}]^\top  \dot p + [\nabla_z H_{\tt CL}]^\top \dot z + [\nabla_q H_{\tt CL}]^\top \dot q \nonumber \\
&& + [\nabla_\phiv H_{\tt CL}]^\top  \dot \phiv \nonumber \\
&=& [\nabla_p H_{\tt CL}]^\top  \left[ - \nabla_q H_{\tt CL} - D_e \nabla_p H_{\tt CL} + D_e \nabla_z H_{\tt CL} \right. \nonumber \\
&& \left. + \tau_e \right] + [ \nabla_z H_{\tt CL}]^\top \left[ - \nabla_\phiv H_{\tt CL} + D_e \nabla_p H_{\tt CL} \right. \nonumber \\
&& \left. - D_e \nabla_z H_{\tt CL} \;+\; \tau_u \right] + [\nabla_q H_{\tt CL}]^\top \nabla_p H_{\tt CL} \nonumber \\
&&  + [ \nabla_\phiv H_{\tt CL} ]^\top   \nabla_z H_{\tt CL}  \nonumber \\
&=&-  [\nabla_p H_{\tt CL}]^\top  D_e \nabla_p H_{\tt CL} + 2 [\nabla_p H_{\tt CL}]^\top  D_e  \nabla_z H_{\tt CL} \nonumber \\
&& -  [ \nabla_z H_{\tt CL}]^\top  D_e \nabla_z H_{\tt CL} + [\nabla_p H_{\tt CL}]^\top  \tau_e \nonumber \\
&=&-  \left[ \hspace{-2mm} \begin{array}{cc} \nabla_p H_{\tt CL} & \nabla_z H_{\tt CL} \end{array} \hspace{-2mm} \right]^\top \left[ \hspace{-2mm} \begin{array}{cc} D_e & -D_e \\ -D_e & D_e \end{array} \hspace{-2mm} \right]  \left[ \hspace{-2mm} \begin{array}{cc} \nabla_p H_{\tt CL} \\ \nabla_z H_{\tt CL} \end{array} \hspace{-2mm} \right] \nonumber \\
&& + [\nabla_p H_{\tt CL}]^\top  \tau_e +  [ \nabla_z H_{\tt CL}]^\top \; \tau_u  \nonumber \\
&\leq& \dot q^\top \tau_e \;+\;  \dphiv^\top \; \tau_u ,
\endequarr
which implies the system \eqref{phcl} is passivity with inputs $(\tau_e,\tau_u)$ and outputs $(\dot q,\dphiv)$.
\end{proof}

\begin{remark}\label{remparam} Note that, conversely to \eqref{kflin}-\eqref{khlin}, the triplet $(J_e,K_e,D_e)$ of the desired mass matrix, desired stiffness matrix and the desired damping matrix  can be written as function of the gains $K_F$ and $K_G$ as follows
\begequarr
J_e &=& (K_F+K_G+I_n)^{-1}(J-K_FM), \\ 
K_e &=& KJ^{-1}(J-K_FM), \\
D_e &=& DJ^{-1}(J-K_FM),
\endequarr
which univocally relate the gains of the controller \eqref{impcontrlin} and the parameters of the closed-loop pH system \eqref{phcl}. The pH form of the closed loop provides a clear physical interpretation of its dynamics, and how the desired mass, stiffness and dissipation matrices, ie. the triple $(J_e,K_e,De)$, affect the control gains $(K_F,K_G,K_H)$.
\end{remark}

A representation of the control system structure is shown in Figure~\ref{figcontrolsystem}, where the gravitational forces have been neglected since we illustrate the motion in the horizontal plane. In this case, the controller is computed by using the feedback signals of the joints and external forces (or torques). 
\begin{figure}[htbp]
\begin{center}
\includegraphics[scale=.33]{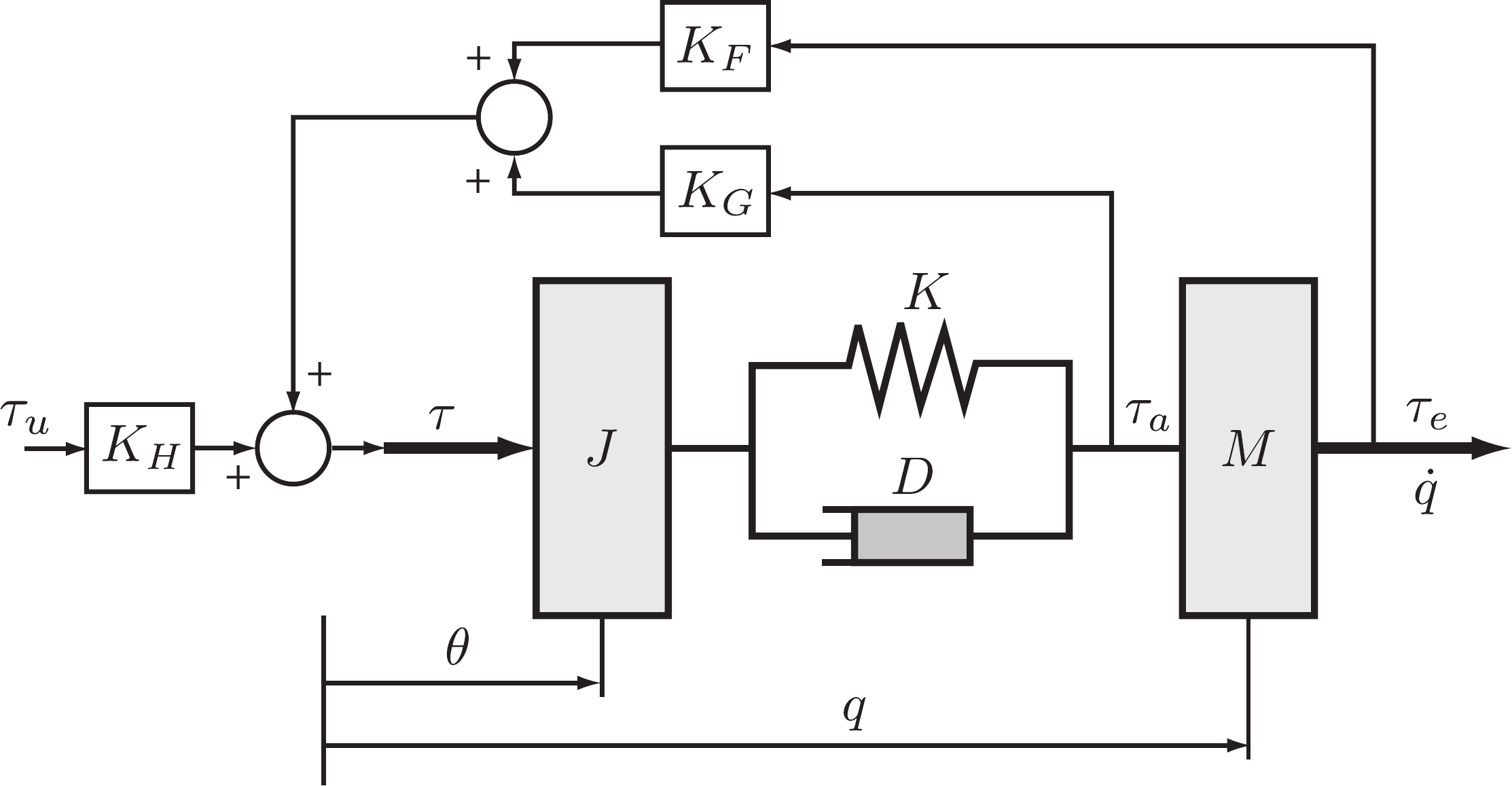}
\caption{Force feedback structure of the control system.}
\label{figcontrolsystem}
\end{center}
\end{figure}

The closed-loop system \eqref{phcl} can be equivalently written in the Euler-Lagrange form as follows
\begequarr \label{elcl1}
\hspace{-6mm} M \ddot{q} - K_e (\phiv-q) - D_e ( \dot{\phiv} -\dot{q} ) +\nabla_qV(q) \hspace{-2mm} &=& \hspace{-2mm} \tau_e \\
 J_e \ddot \phiv + K_e (\phiv-q) + D_e ( \dot{\phiv}-\dot{q}) \hspace{-2mm} &=&\hspace{-2mm}  \tau_u, \label{elcl2}
\endequarr
which can be thought to describe the motion equation of a mechanical system, where $J_e$, $K_e$ and $D_e$ are the shaped inertia, stiffness and dissipation matrices. Figure~\ref{figmeccl} shows a diagram of the close-loop system. Note that, due to the coordinate change $\theta \to \phiv$ given in \eqref{coc1}, the dynamics of open-loop and closed-loop systems are written in different coordinates.
\begin{figure}[htbp]
\begin{center}
\includegraphics[scale=.33]{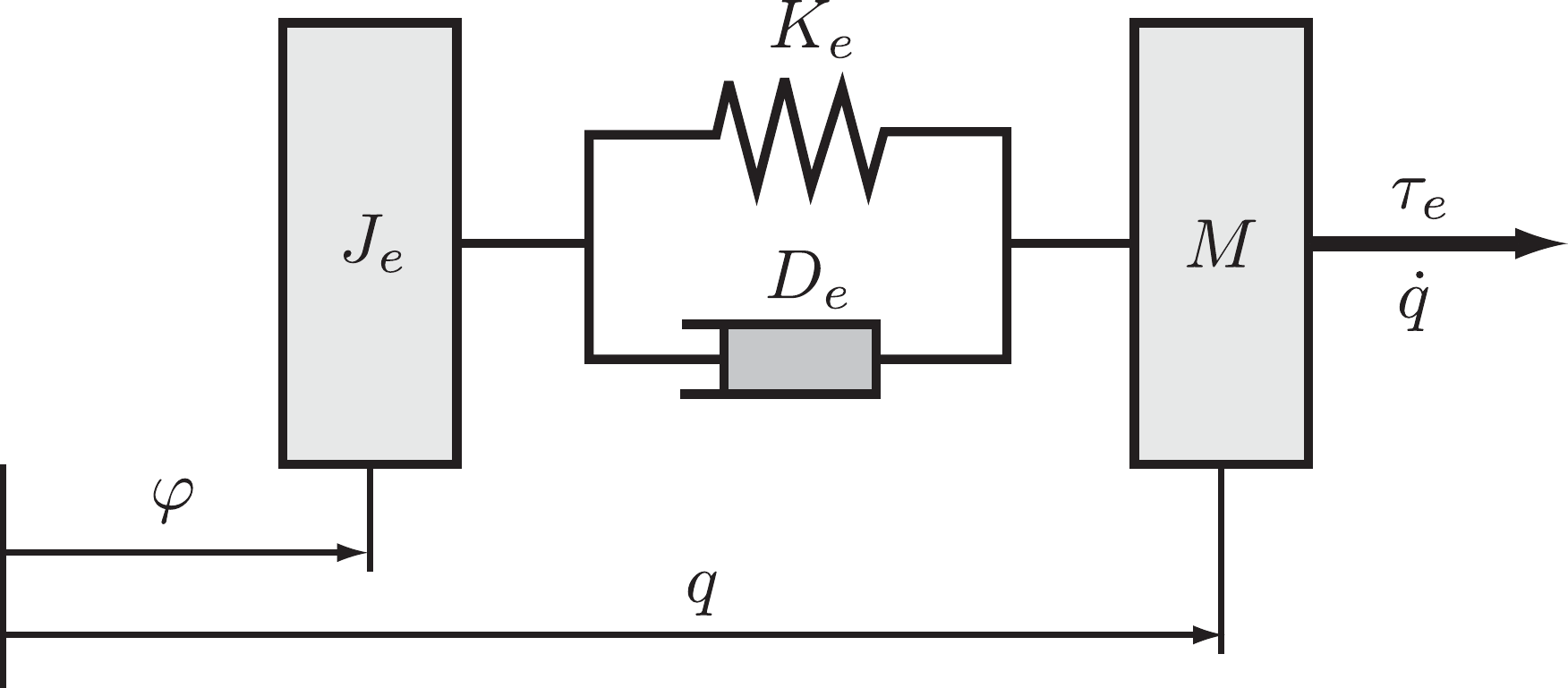}
\caption{Equivalent mechanical dynamics of the closed-loop system.}
\label{figmeccl}
\end{center}
\end{figure}

\begin{remark} We notice that the classical impedance control using force feedback for LTI SISO systems (see e.g. \cite{Colgate1989}) can be derived from proposition \ref{proplin}. The result discussed in \cite{Colgate1989} ensures that the system \eqref{mechl1}-\eqref{mechl4} with $n=1$ in closed loop with the controller 
$$
\tau=K_F \tau_e,
$$ 
with $K_F$ satisfying $-1<K_F< \frac J M$, defines a passive mapping $\tau_e \to \dot q$. Indeed, consider the control law \eqref{impcontrlin} with $K_G=0$. Then, using \eqref{kflin} and \eqref{kglin} we obtain
\begin{eqnarray}
J_e &=& \frac{(J-K_FM)}{K_F+1}, \quad
K_e = \frac{K}{J}(J-K_FM), \nonumber \\
D_e &=& \frac{D}{J}(J-K_FM), \nonumber
\end{eqnarray}
which satisfies $J_e,\,K_e,\,D_e\,>0$ iff $-1<K_F< \frac J M$. In addition, the closed-loop dynamics can be equivalently written in Euler-Lagrange form \eqref{elcl1}-\eqref{elcl2}.
However, notice that by selecting a non-zero matrix $K_G$ there is more freedom to choose the inertia, the stiffness and the dissipation matrices $J_e$, $K_e$ and $D_e$. These extra degree of freedom is achieved at the expense of using the measurement of joint torque $\tau_a$.
\end{remark}
\subsection{Robot-human/environment interaction properties.}
In this section, we present a stability analysis of a controlled flexible robot interacting similar to the analyses in \cite{Lacevic2011}. The analysis in \cite{Lacevic2011} considers the robot dynamics \eqref{mechl1}-\eqref{mechl4} with $n=1$ in closed loop with a PI controller and coupled with a dynamic system that represent the human or the environment. The impedance of the human/environment is assumed to be 
\begequarr \label{human1dof}
Z_h(s) = \frac{ \mathcal{L}[ \tau_h(t)]} {\mathcal{L}[\dot q(t)]} = M_h s  + D_h + \frac{K_h}{s}.
\endequarr 
Using this model, the authors in \cite{Lacevic2011} studied the stability in case of parameter uncertainties in the model \eqref{human1dof}. Thus, the mass, stiffness and dissipation coefficients have an uncertainty quantity as follows: $M_h=\bar M_h + \Delta_{M_h}$, $K_h=\bar K_h + \Delta_{K_h}$ and $D_h=\bar D_h + \Delta_{D_h}$. In addition, it is also assumed that $ M_h \in [0,M_{h}^{\max}]$, $K_h \in [0,K_{h}^{\max}]$ and $D_h \in [0,D_{h}^{\max}]$. It is shown in \cite{Lacevic2011} that, for the one-dimensional case, there always exists a PI controller on the external force that stabilise the system.

We consider now the controlled robot \eqref{phcl} and we compute its admittance $Y(s) = \frac{ \mathcal{L}[\dot q(t)] }{\mathcal{L}[\tau_e(t)]}$ as follows
\begequarr
Y(s) = \frac{J_e s^2 + D_e s + K_e}{s\left[ J_e M s^2 + D_e (J_e+M) s + K_e (J_e+M) \right]}.
\endequarr 
Under the same scenario as in \cite{Lacevic2011}, that is with the impedance of the human/environment as in \eqref{human1dof}, the interconnection of the robot and the human/environment results passive, thus stable. In what follows, we will extend this result and generalise it for the study the multi-dimensional case. To do that, we write, with some abuse of notation, the dynamics of the human/environment as\footnote{The abuse of notation comes from the fact that, to simplify the notation, we use the same parameters for the 1-dimensional case and for the n-dimensional case.}
\begequarr \label{humanndof}
\tau_h = M_h \ddot{q}  + D_h \dot{q} + K_h q.
\endequarr 
It is straightforward to show that the dynamics of the controlled robot \eqref{phcl} in physical interaction with the human/environment dynamics \eqref{humanndof} can be written in the form of Euler-Lagrange equations as follows
\begequarr
\hspace{-7mm} (M\!+\!M_h) \ddot{q}\! -\! K_e (\phiv-q) \!+ \!D_e ( \dot{\phiv} - \dot{q}) \!+ \!D_h \dot q \!+\! K_h q  \hspace{-3mm} &=&\hspace{-3mm} 0 \label{rhint1} \\
J_e \ddot \phiv \!+\! K_e (\phiv-q) \!+ \!D_e (\dot{\phiv}-\dot{q}) \hspace{-3mm}&=&\hspace{-3mm} \tau_u. \label{rhint2}
\endequarr
The interconnection of the controlled robot and the human/environment shown in Figure~\ref{figinteract} is power preserving, and therefore, the system \eqref{rhint1}-\eqref{rhint2} is passive provided that $M_h$, $K_h$, and $D_h$ are positive (semi-)definite. The additional control input $\tau_u$ may be used to ensure asymptotic stability of a desired equilibrium if required.
\begin{figure}[htbp]
\begin{center}
\includegraphics[scale=.315]{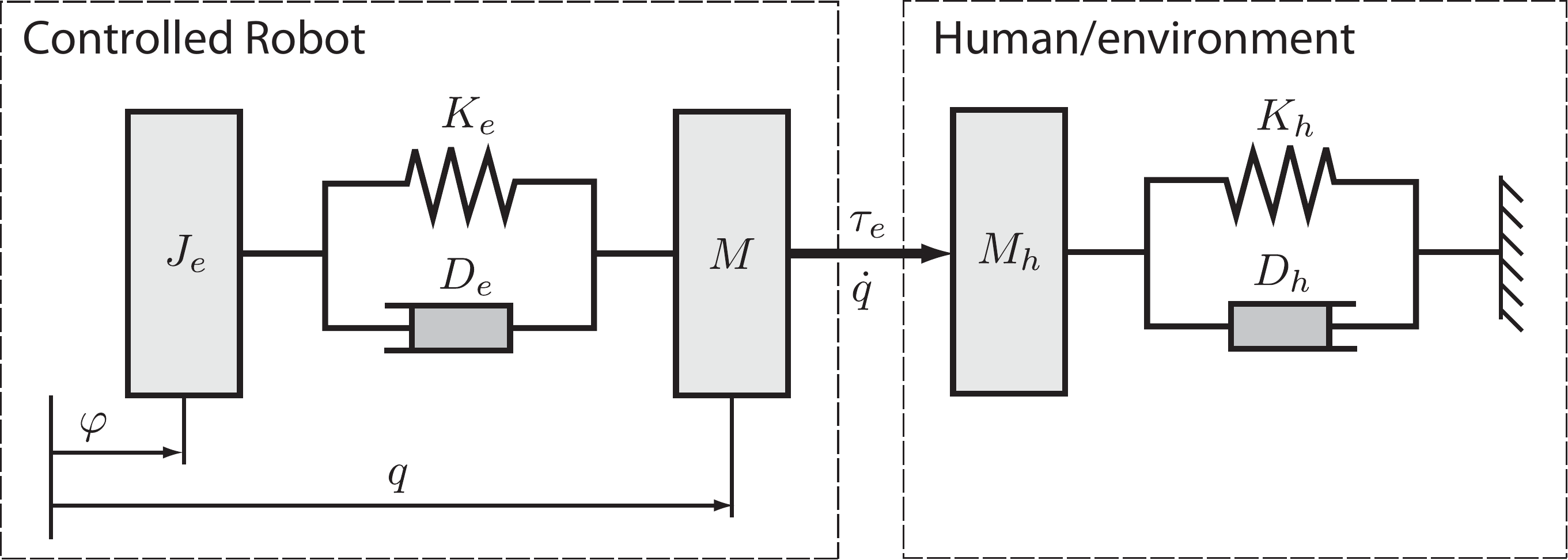}
\caption{Robot-human/environment interaction.}
\label{figinteract}
\end{center}
\end{figure}
\subsection{Extension to nonlinear multivariable systems.}
In this section, we extend the result of Proposition \ref{proplin} to the nonlinear case in which the mass matrix is a function of the coordinates. That is, we drop assumption {\bf A1}.

We consider the mass matrix as function of the coordinates $M(q)$, then \eqref{el1}--\eqref{el2} takes the pH form as follows\footnote{Since it is clear that in this part of the work the matrix $M$ is a function of $q$, and to simplify the notation, we do not write explicitly this dependency in the equations.}
\begequarr
\dq &=& M^{-1} p \label{mechnl1}\\ 
\dtheta &=& J^{-1} s \label{mechnl2} \\
\dotp &=& -\nabla_qV(q) + K (\theta-q) + D (J^{-1}s-M^{-1}p) \nonumber \\
&& - \frac 12 \nabla_q(p^\top M^{-1}p) +\tau_e  \label{mechnl3}  \\
\ds &=& - K (\theta-q) - D (J^{-1}s-M^{-1}p) +\tau \label{mechnl4}
\endequarr

\begin{prop} \label{propnonlin} Consider the robot dynamics \eqref{mechnl1}--\eqref{mechnl4} in closed loop with the control law
\begequarr 
\tau&=&K_F \; \tau_e - K_G \; \tau_{a} -  K_F\; C(q,\dq) \; \dot q \;  - K_F \nabla_qV(q)  \nonumber \\
&&  \;+\; K_H\;\tau_u , \label{impcontrnonlin}
\endequarr
where
$\tau_{a}$, $K_F$, $K_G$ and $K_H$ are given in \eqref{taua}, \eqref{kflin}, \eqref{kglin} and \eqref{khlin}, respectively, and $\tau_u$ is an input available for further control loops.
The constant matrices $J_e=J_e^\top>0$ and $K_e=K_e^\top>0$ are free to be chosen such that $D_e=DK^{-1}K_e$ is symmetric and positive definite. Then,
\begite
\item[$i).$] The closed loop dynamics can be written in pH form as follows
\begequarr \label{nphcl}
\left[ \begarr{c} \dq \\ \dphiv \\ \dotp \\ \dz \endarr \right] \hspace{-3mm}&=& \hspace{-3mm} \left[ \begarr{cccc} 0_{n\times n} & 0_{n\times n} & I_{n\times n} & 0_{n\times n} \\ 0_{n\times n} & 0_{n\times n} & 0_{n\times n} & I_{n\times n} \\ -I_{n\times n} & 0_{n\times n} & -D_e & D_e \\ 0_{n\times n} & -I_{n\times n} & D_e & -D_e \endarr \right] \nabla H_{\tt NCL} \nonumber \\
&& + \left[ \begarr{c} 0_{n\times 1} \\ 0_{n\times 1} \\ \tau_e \\ \tau_u \endarr \right]
\endequarr
with
\begequarr \label{nhcl}
H_{\tt NCL}&=&\frac 12 p^\top M^{-1}(q) p + \frac 12 z^\top J_e^{-1} z + \frac 12 (\phiv-q)^\top K_e \nonumber \\
&& \times (\phiv-q) +V(q).
\endequarr

The new generalised position $\phiv$ and momentum $z$ are obtained using the transformations
\begequarr
\hspace{-8mm} \phiv \hspace{-2mm}&=& \hspace{-2mm} K_e^{-1} (K_e-K)q+K_e^{-1}K\theta  \label{ncoc1} \\
\hspace{-8mm} z \hspace{-2mm}&=& \hspace{-2mm} J_eK_e^{-1}(K_e-K)M^{-1}p+J_eK_e^{-1}KJ^{-1}s \label{ncoc2}
\endequarr

\item[$ii).$] The closed loop is passive with inputs $(\tau_e,\tau_u)$, outputs $(\dq,\dphiv)$ and storage function $H_{\tt CL}$.
\endite
\end{prop}
\begin{proof} The proof of this proposition follows a similar procedure as the proof of Proposition \ref{proplin}. Since the presence of the nonlinear terms makes the proof much longer due to algebraic lengthy calculations and to comply with the space limitation, we do not repeat the same procedure here. We highlight that the proof of proposition 2 does not rely in any extra additional technical step other than lengthy calculations.
\end{proof}

\begin{remark} We notice that the controller for the linear and nonlinear cases, \eqref{impcontrlin} and \eqref{impcontrnonlin} respectively, differ only in a term to compensates Coriolis forces. This extra term vanishes when the mass matrix $M$ is constant. Therefore, the result in Proposition \ref{propnonlin} can be seen as a generalisation of Proposition \ref{proplin} and of the classical result for LTI SISO systems \cite{Colgate1989}. The price to pay for this extension is the necessity of state information that is feedback to the controller \eqref{impcontrnonlin}. 
\end{remark}
\subsection{Gravity compensation} 
The controllers in Proposition~\ref{proplin} and Proposition~\ref{propnonlin} are designed such that the closed-loop dynamics preserve the mechanical structure. Therefore, the closed loop \eqref{nphcl} can be equivalently written in the Euler-Lagrange form
\begequarr 
M(q) \ddot{q}\!+\!C(q,\dot q) \dot q\! - \!K_e (\phiv-q) \!- \!D_e ( \dot{\phiv}-\dot{q}) \!+\!\nabla_qV(q) \hspace{-3mm}&=& \hspace{-3mm} \tau_e \nonumber \\ 
J_e \ddot \phiv \!+ \!K_e (\phiv-q) \!+\! D_e (\dot{\phiv}-\dot{q}) \hspace{-3mm} &=& \hspace{-3mm} \tau_u, \nonumber 
\endequarr
which is equivalent to the dynamic of a mechanical system with $J_e$, $K_e$ and $D_e$ being the shaped inertia, stiffness and dissipation matrices. The matrix $C(q,\dot q)$ describes the Coriolis forces and $V(q)$ the potential function. The structure of the dynamics is the same as the structure of the mechanical system in \cite{Ott2008}. Therefore, we can use the gravity compensation proposed in that work and readily added it to our controller using the control law
\begin{equation} \label{gcomp}
\tau_u=-K_\phiv (\phiv-\phiv_d) - D_\phiv \dot \phiv + \bar g (\phiv),
\end{equation}
where $\phiv_d$ is the desired positon, $K_\phiv$ and $D_\phiv$ are constant positive matrices to be chosen, and $\bar g$ is the gravity compensation. Notice that the first two terms in \eqref{gcomp} are added to help shaping the admittance (see further details in \cite{Ott2008}).
\section{NUMERICAL ANALYSIS}\label{secna}
\subsection{Constant mass matrix} 
In this section, we investigate the effect of the control gains on the stability of the full-closed loop and its approximation to the desired admittance via Bode-diagram and the pole-zero maps. For this numerical analysis, we consider the system \eqref{el1}-\eqref{el2} in closed loop with the controller \eqref{impcontrlin} in Proposition \ref{proplin}, with $\tau_u$ as in \eqref{gcomp}. Since the motion is horizontal, the gravity forces are not considered. The parameters of the system with $n=1$ are $M=3$Kg, $J=3$Kg, $K=10^6$Nm$^{-1}$ and $D=1$Nm$^{-1}$s. The parameters of the desired system \eqref{targdyn} are $M_d=3$Kg, $K_d=10$Nm$^{-1}$ and $D_d=100$Nm$^{-1}$s. The gains of the controller are $K_\varphi=100$Nm$^{-1}$, $D_\varphi=10$Nm$^{-1}$s, whilst remaining gains take the following values $K_F=\{-0.9,0,0.9 \}$ and  $K_G=\{0,1,4 \}$. The change in $K_F$ and $K_G$ allow us to analysis their effect in the closed loop.

Figures~\ref{bodekf-09}-\ref{bodekf09} shows that the frequency response of the closed loop approximates better the frequency response of the desired admittance when the values of the gain $K_F$ are chosen towards $\frac{J}{M}=1$. We recall that $-1<K_F<\frac{J}{M}$ to ensure that the closed loop is passive \cite{Colgate1989}. It can also be seen that an increment on the gain $K_G$ improves the closed-loop performance. Indeed, it is clear in Figures~\ref{bodekf-09}, Figure~\ref{bodekf00} and Figure~\ref{bodekf09} that the frequency response of the closed-loop admittance approaches the frequency response of the desired admittance as $K_G$ increases. 
\begin{figure}[htbp]
\begin{center}
\includegraphics[scale=.45]{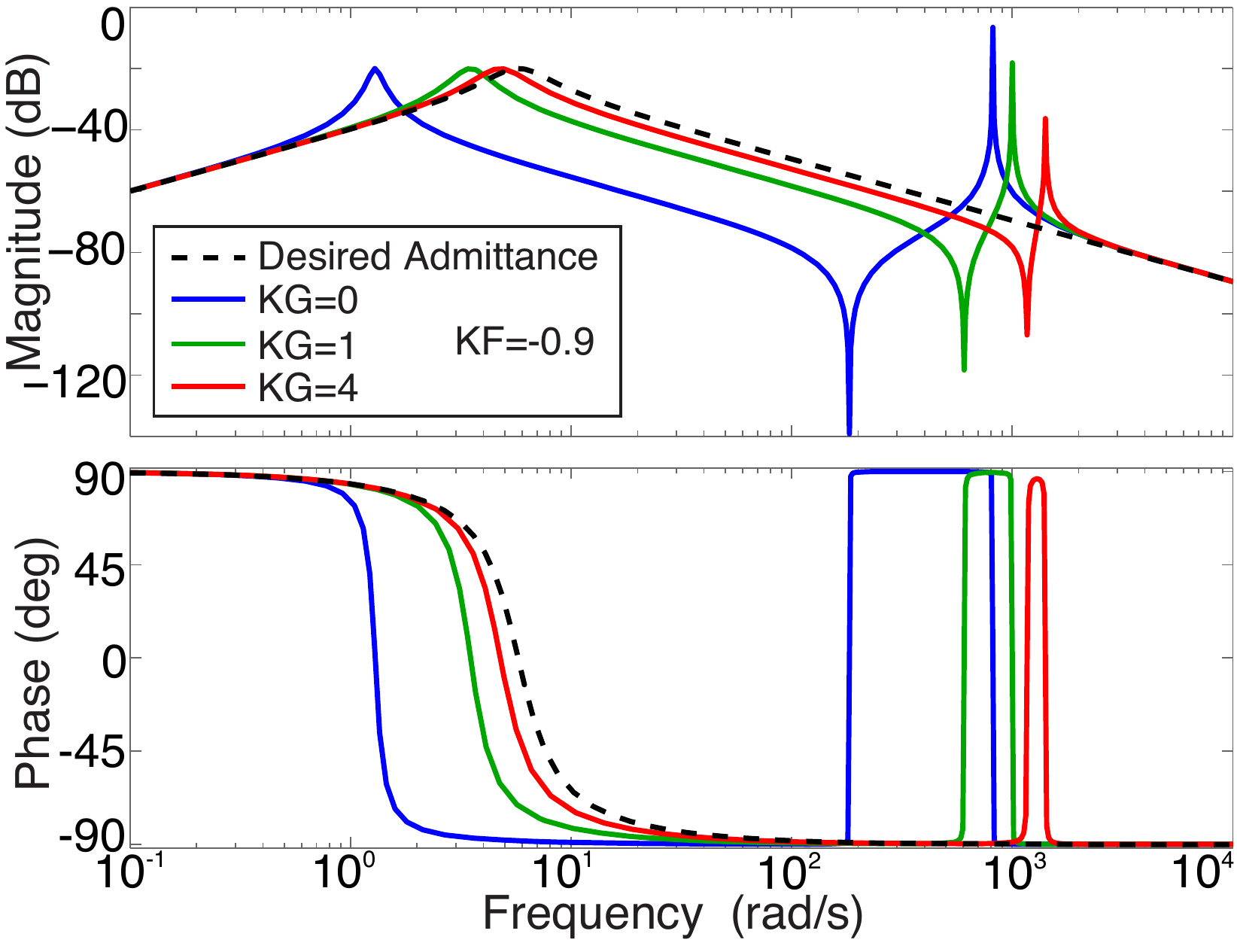}
\caption{Closed-loop Bode diagram for different control gains.}
\label{bodekf-09}
\end{center}
\end{figure}
\begin{figure}[htbp]
\begin{center}
\includegraphics[scale=.45]{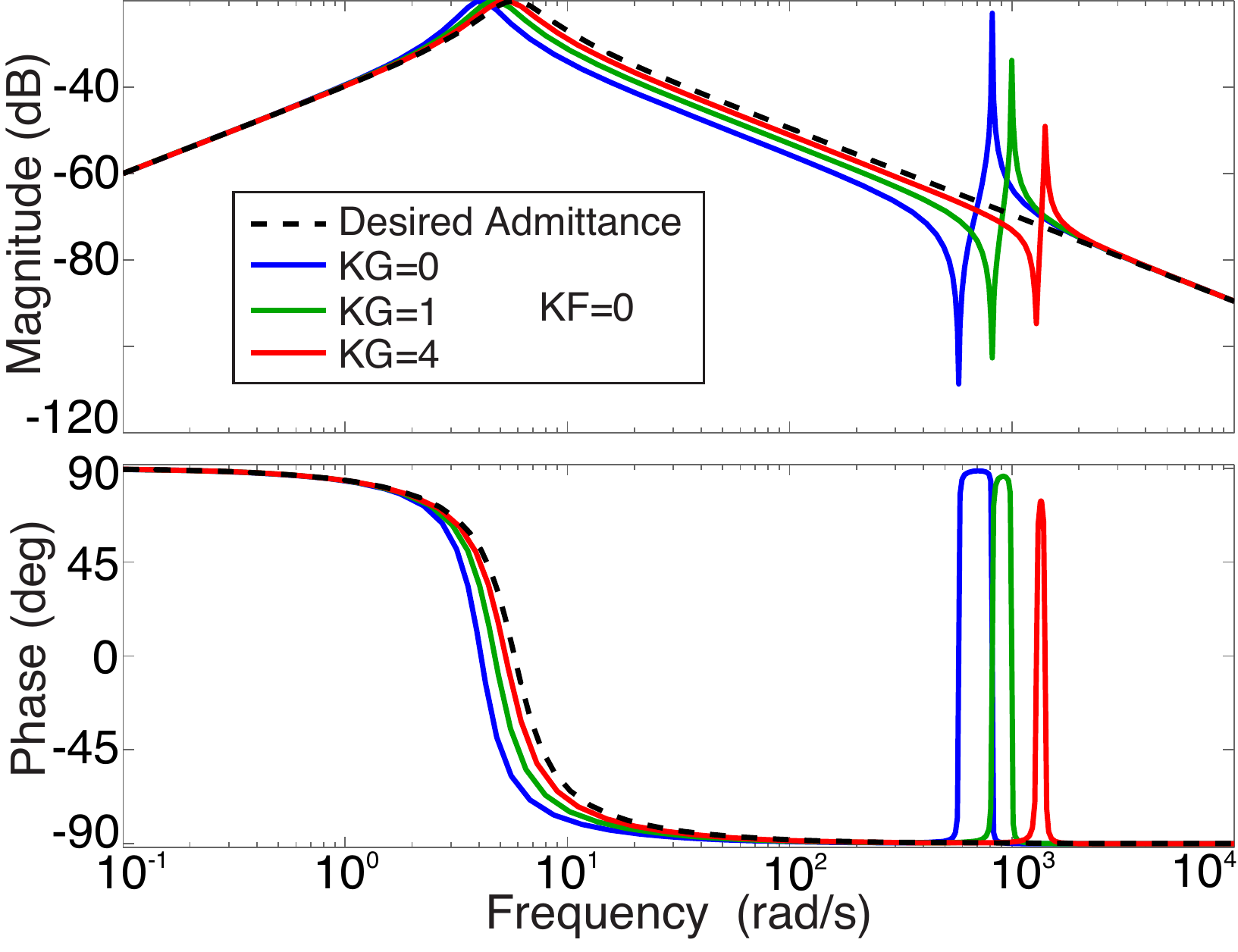}
\caption{Closed-loop Bode diagram for different control gains.}
\label{bodekf00}
\end{center}
\end{figure}
\begin{figure}[htbp]
\begin{center}
\includegraphics[scale=.45]{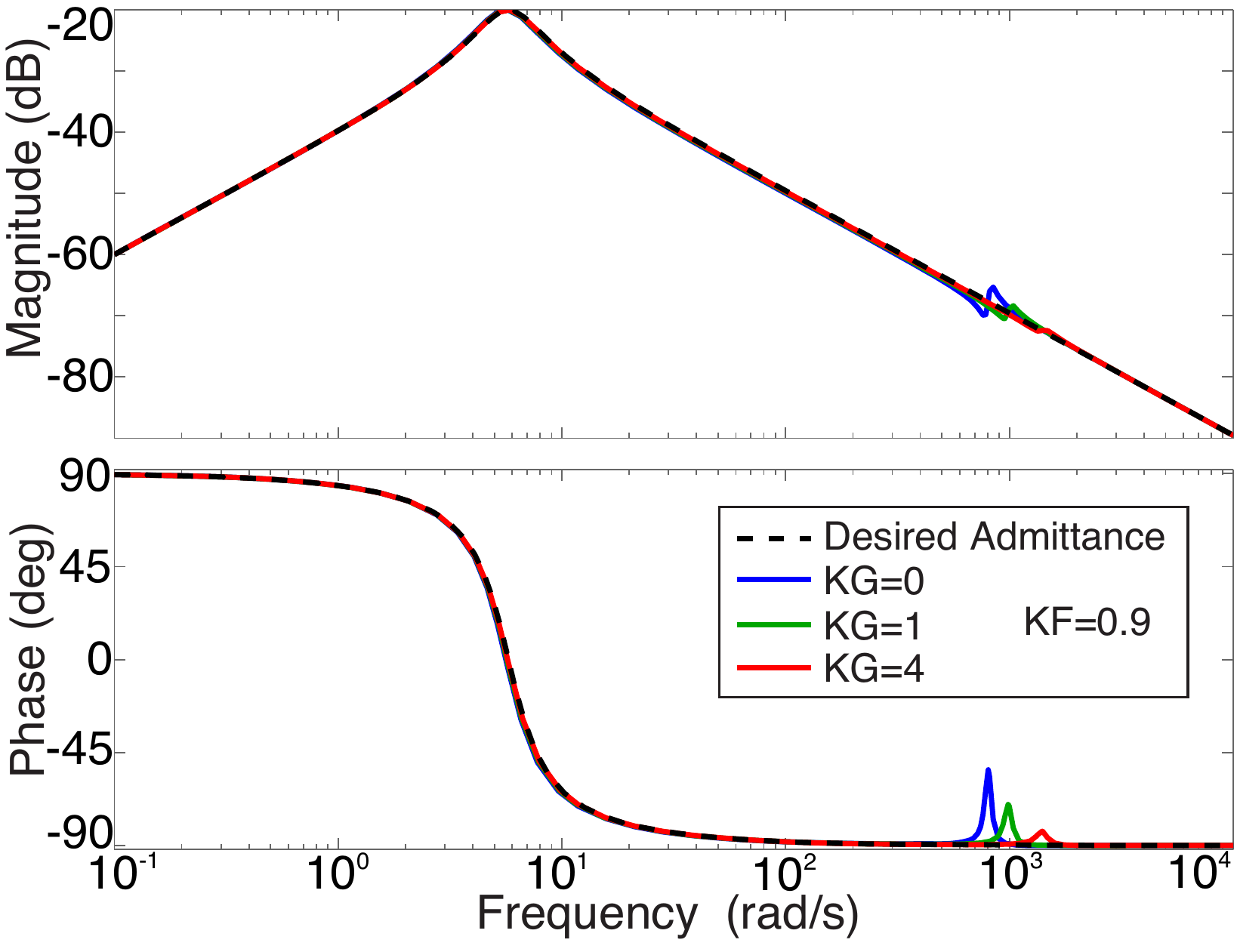}
\caption{Closed-loop Bode diagram for different control gains.}
\label{bodekf09}
\end{center}
\end{figure}
Figure~\ref{zpandzoom} shows the pole-zero map of the closed loop for different values of the gains $K_F$ and $K_G$. The full view of the pole-zero map is shown in the upper plot of Figure~\ref{zpandzoom}, while the bottom plot show a zoom around the origin to better visualise the pole distribution in this region. The desired admittance \eqref{targadm} has two poles and a zero, which are displayed in the bottom plot of Figure~\ref{zpandzoom}. In the same plot illustrate that the pole of the closed-loop admittance approaches the poles of the desired admittance when both $K_F$ and $K_G$ increase. Also, the upper plot of Figure~\ref{zpandzoom} evidence that the extra zero and pole that appear in the closed-loop admittance are better compensated when both $K_F$ and $K_G$ increase. We can infer that the feedback of both the external and join forces enhance the performance of the close loop.
\begin{figure}[htbp]
\begin{center}
\includegraphics[scale=.35]{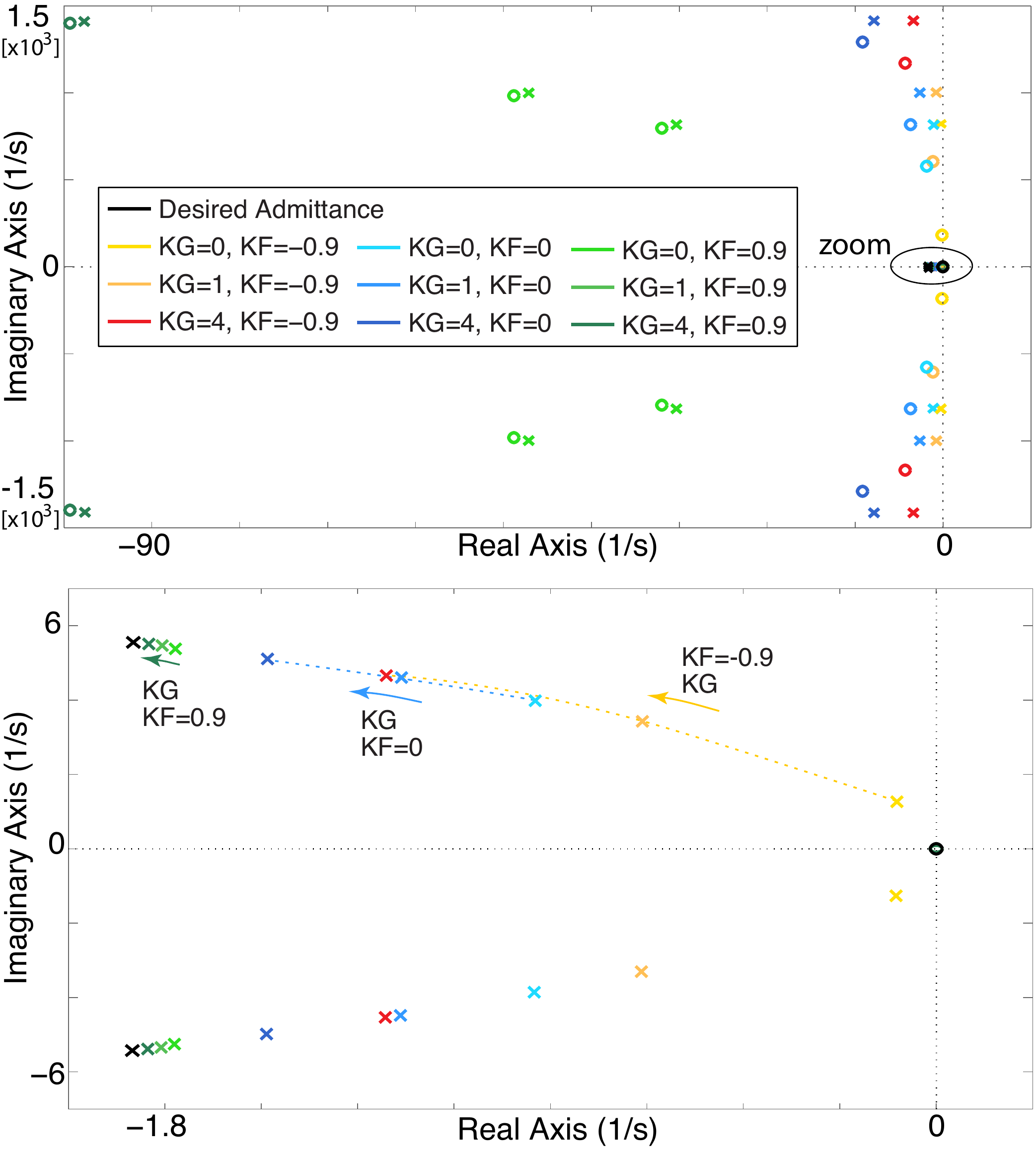}
\caption{Robot-environment interaction.}
\label{zpandzoom}
\end{center}
\end{figure}
\subsection{Non-constant mass matrix}

In this section, present simulation of a nonlinear model of a robotic manipulator to show the performance of the controller presented in Proposition 2. We consider the dynamic model of the KUKA-LWR4+ Robot, whose dynamics  can be described by equations \eqref{el1}--\eqref{el2} with $D=0$ (see for example \cite{Gaz2014}). We use the control law \eqref{impcontrnonlin} of Proposition 2, which is suitable for nonlinear dynamics, with $\tau_u$ as in \eqref{gcomp}.

In this scenario, we consider the target dynamics
\begequarr
M(q) \ddot{q} + (C(q,\dot{q})+D_{\theta})  \dot{q} + K_{\theta} (q - q_d) = \tau_e
\label{targetdyn}
\endequarr
The desired stiffness and damping matrices are $K_{\theta}= \text{diag} \lbrace 1000,1000,1000,1000,1000,1000\rbrace$ N$\cdot$m/rad and $D_{\theta} =  \text{diag} \lbrace 135, 135, 135 ,135, 13.5, 13.5 ,13.5 \rbrace$  N$\cdot$m$\cdot$s/rad.

The controller is tuned to obtain the performance that approaches the target dynamics, and we selected the following values: $K_e = 2 K$, $D_e = DK^{-1}K_e = 0$, $K_{\varphi} = K_{\theta} $ and $D_{\varphi}=D_{\theta}$.

The simulation shows the response of the closed-loop to a step-wise excitation in the external torque, $\tau_e=10$ n$\cdot$m, exerted on the second joint. Figure~\ref{fig:don_target_q} shows the time history of the coordinate of the second joint for different values of the control gain $J_e$. The figure also shows the time history of the second coordinate of the target dynamics \eqref{targetdyn}. The simulation shows that the used of smaller values of $J_e$ (equivalently larger values of $K_G$) respect to the original inertia of the motor $J$ improves the time response in regard to the target dynamics. This is in agreement with the analysis for linear system in previous section. Figure~\ref{fig:don_target_torque} show the only time history of the control torque in the second joint, which is sufficiently smooth and does not present excessive values. 
\begin{figure}[htbp]
\begin{center}
\includegraphics[scale=.65]{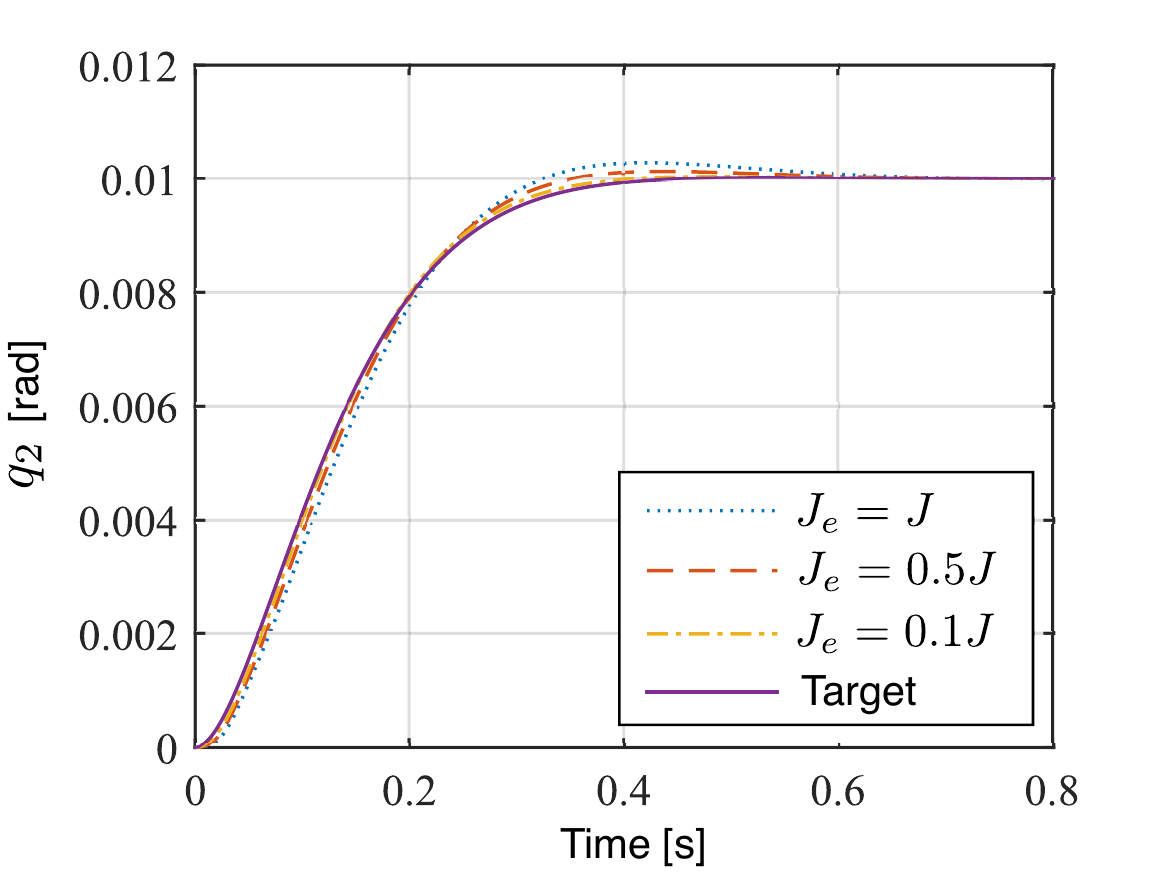}
\caption{Time history of the coordinate of the second joint .}
\label{fig:don_target_q}
\end{center}
\end{figure}
\begin{figure}[htbp]
\begin{center}
\includegraphics[scale=.65]{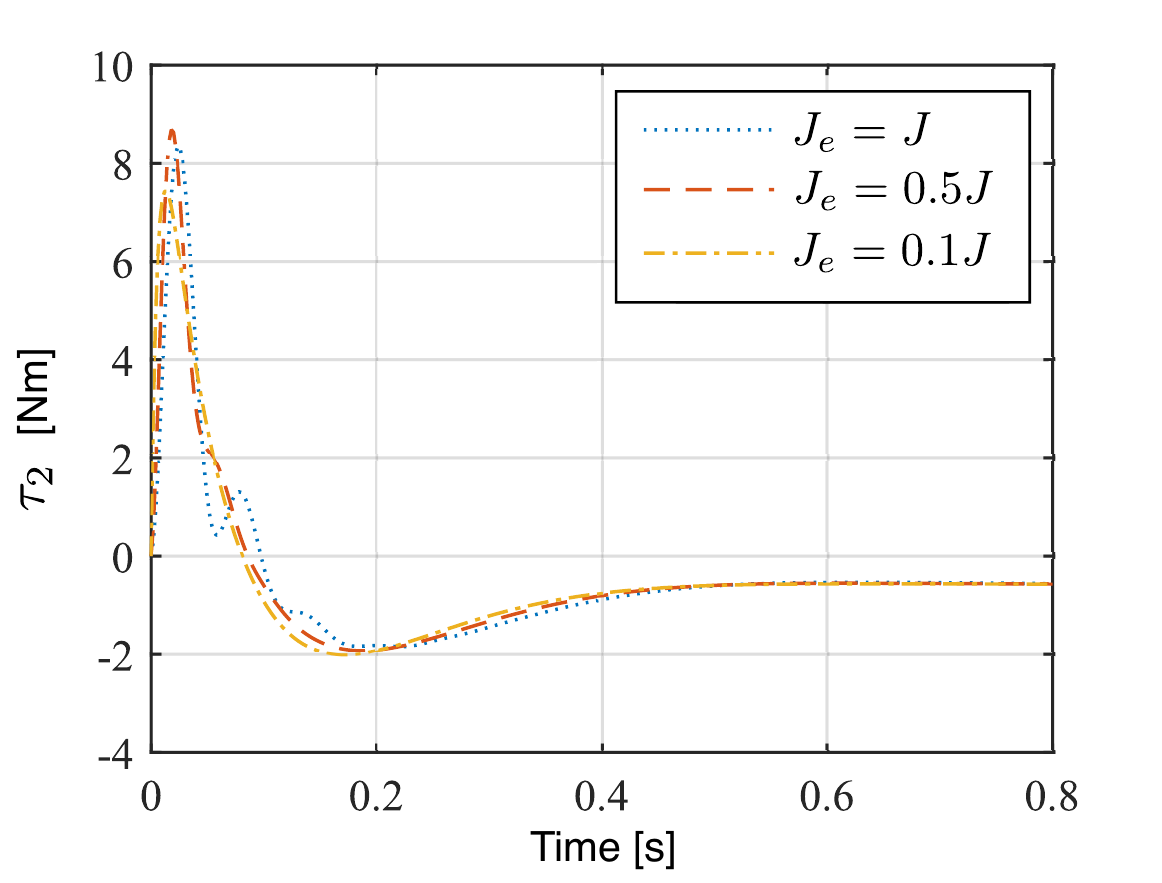}
\caption{Time history of the control torque of the second joint.}
\label{fig:don_target_torque}
\end{center}
\end{figure}

\section{CONCLUSIONS}\label{conclusion}
We present an impedance control design for multi-variable linear and nonlinear robotic systems. The proposed design takes advantage of the external force and state measurements. This allows the controller for extra gains and flexibility that allow to approximation of the desired admittance or target dynamics. Additionally, we show that the closed loop system ensures passivity of the robot for the input-output pair given by the external torque and the link velocities. We also present a numerical analysis and simulation results that show performance of the controller and the effect of the tuning gains in the time response.





\bibliographystyle{IEEEtran}        
\bibliography{impedcontrol}           

\begin{thebibliography}{10}
\providecommand{\url}[1]{#1}
\csname url@rmstyle\endcsname
\providecommand{\newblock}{\relax}
\providecommand{\bibinfo}[2]{#2}
\providecommand\BIBentrySTDinterwordspacing{\spaceskip=0pt\relax}
\providecommand\BIBentryALTinterwordstretchfactor{4}
\providecommand\BIBentryALTinterwordspacing{\spaceskip=\fontdimen2\font plus
\BIBentryALTinterwordstretchfactor\fontdimen3\font minus
  \fontdimen4\font\relax}
\providecommand\BIBforeignlanguage[2]{{%
\expandafter\ifx\csname l@#1\endcsname\relax
\typeout{** WARNING: IEEEtran.bst: No hyphenation pattern has been}%
\typeout{** loaded for the language `#1'. Using the pattern for}%
\typeout{** the default language instead.}%
\else
\language=\csname l@#1\endcsname
\fi
#2}}

\bibitem{Hogan1985}
N.~Hogan, ``Impedance control: An approach to manipulation,'' \emph{ASME
  Journal of Dyanmic Systems, Measurements, and Control}, vol. 107, no.~1, pp.
  1--24, 1985.

\bibitem{Chiaverini1999}
S.~Chiaverini, B.~Siciliano, and L.~Villani, ``A survey of robot interaction
  control schemes with experimental comparison,'' \emph{IEEE/ASME Transactions
  on Mechatronics}, vol.~4, no.~3, pp. 273--285, 1999.

\bibitem{Siciliano1999}
B.~Siciliano and L.~Villani, \emph{Robot Force Control}.\hskip 1em plus 0.5em
  minus 0.4em\relax Kluwer Academic Publishers, 1999.

\bibitem{Buerger2007}
S.~P. Buerger and N.~Hogan, ``Complementary stability and loop shaping for
  improved human-robot interaction,'' \emph{IEEE Transactions on Robotics},
  vol.~23, no.~2, pp. 232--244, 2007.

\bibitem{Schaffer2007}
A.~Albu-Sch\"{a}ffer, C.~Ott, and G.~Hirzinger, ``A unified passivity-based
  control framework for position, torque and impedance control of flexible
  joint robots,'' \emph{International Journal of Robotics Research}, vol.~26,
  no.~1, pp. 23--39, 2007.

\bibitem{Calanca2017}
A.~Calanca and P.~Fiorini, ``Impedance control of series elastic actuator based
  on well-defined force dynamics,'' \emph{Robotics and Autonomous Systems},
  vol.~96, no.~10, pp. 81--92, 2017.

\bibitem{Magrini2015}
E.~Magrini, F.~Flacco, and A.~D. Luca, ``Control of generalized contact motion
  and force in physical human-robot interaction,'' in \emph{IEEE International
  Conference on Robotics and Automation}, Seattle, USA, 2015.

\bibitem{Ficuciello2015}
F.~Ficuciello, L.~Villani, and B.~Siciliano, ``Variable impedance control of
  redundant manipulators for intuitive human-robot physical interaction,''
  \emph{IEEE Transactions on Robotics}, vol.~31, no.~4, pp. 850--863, 2015.

\bibitem{Colgate1989}
E.~Colgate and N.~Hogan, ``The interaction of robots with passive environments:
  Application to force feedback control,'' in \emph{Advanced Robotics}, K.~J.
  Waldron, Ed.\hskip 1em plus 0.5em minus 0.4em\relax Berlin: Springer-Verlag,
  1989, pp. 465--474.

\bibitem{Ott2008}
C.~Ott, A.~Albu-Sch\"{a}ffer, A.~Kugi, and G.~Hirzinger, ``On the
  passivity-based impedance control of flexible joint robots,'' \emph{IEEE
  Transactions on Robotics}, vol.~24, no.~2, pp. 416--429, 2008.

\bibitem{Ortega2017}
R.~Ortega, A.~Donaire, and J.~G. Romero, ``Passivity-based control of
  mechanical systems,'' in \emph{Feedback Stabilization of Controlled Dynamical
  Systems--In Honor of Laurent Praly}, N.~Petit, Ed.\hskip 1em plus 0.5em minus
  0.4em\relax Springer Berlin/Heidelberg, 2017, vol. 472, ch.~7, pp. 167--199.

\bibitem{Lacevic2011}
B.~Lacevic and P.~Rocco, ``Closed-form solution to controller design for
  human-robot interaction,'' \emph{ASME Journal of Dyanmic Systems,
  Measurements, and Control}, vol. 133, no.~2, pp. 24\,501/1--7, 2011.

\bibitem{Gaz2014}
C.~Gaz, F.~Flacco, and A.~D. Luca, ``Identifying the dynamic model used by the
  {KUKA LWR}: A reverse engineering approach,'' in \emph{IEEE International
  Conference on Robotics and Automation}, Hong Kond, Chine, 2014.

\end{thebibliography}

\end{document}